\documentclass[onecolumn]{IEEEtran}
\usepackage[numbers]{natbib}
\usepackage{cite}
\usepackage{hyperref}

\usepackage{color}
\usepackage{amsfonts,amsmath,amsthm}

\usepackage{algorithmic}
 
\usepackage{multirow}

\newtheorem{lemma}{Lemma} 
\newtheorem{theorem}{Theorem} 
\newtheorem{definition}{Definition} 
\newtheorem{corollary}{Corollary}

\newcommand{\opt}{\mathrm{opt}}
\newcommand{\non}{\mathrm{non}}
 
\begin{document}

\title{Mixed Strategy May Outperform Pure Strategy: An Initial Study}
\author{Jun He,
\thanks{Jun He   is with Department of Computer Science, Aberystwyth University, Ceredigion, SY23 3DB, UK. Email:  jun.he@aber.ac.uk.}
    \and Wei Hou,  Hongbin Dong, 
    \thanks{Wei Hou and Hongbin Dong are with College of Computer Science and Technology, Harbin Engineering University, Harbin, 150001, China}
   \and Feidun He  
  \thanks{Feidun He is with School of Information Science and Technology, Southwest Jiaotong University, Chengdu, Sichuan, 610031, China}   }

\maketitle

\begin{abstract}
A pure strategy  metaheuristic is one that applies the same search method at each generation of the algorithm. A mixed strategy  metaheuristic is one that selects a search method probabilistically from a set of strategies at each generation. For example, a classical genetic algorithm, that applies mutation with probability 0.9 and crossover with probability 0.1, belong to mixed strategy heuristics. A (1+1) evolutionary algorithm using mutation but no crossover is a pure strategy metaheuristic. The purpose of this paper is to compare the performance between mixed strategy    and  pure strategy metaheuristics.  
The main results of the current paper are summarised as follows.  (1)  We construct two novel mixed strategy evolutionary algorithms for solving the 0-1 knapsack problem. Experimental results show that the mixed strategy algorithms may find better solutions than pure strategy algorithms in up to   77.8\% instances through experiments.   (2) We establish  a sufficient and  necessary condition when  the expected runtime time of mixed strategy metaheuristics is smaller that that of pure strategy mixed strategy metaheuristics.
\end{abstract}

 \section{Introduction}
In the last three decades, metaheuristics  have been widely applied in solving combinatorial optimisation problems \citep{glover2003handbook,gendreau2010handbook}.   Metaheuristics include, but  are not restricted to, Ant Colony Optimization (ACO),  Genetic Algorithms (GA), Iterated Local Search (ILS), Simulated
Annealing (SA), and Tabu Search (TS)  \citep{blum2003metaheuristics}. Different search strategies have been developed in  these metaheuristics. Each search strategy has its own advantage.   Therefore it is a natural idea to combine the advantages of several search strategies together. This leads to hybrid metaheuristics~\citep{blum2008hybrid} such as hyper-heuristic \citep{burke2003hyper} and memetic algorithm \citep{neri2011handbook}.

Mixed strategy metaheuristics~\citep{he2005game} belong to the family of hybrid metaheuristics. They are inspired from    the game theory \citep{dutta1999strategies}. A pure strategy  metaheuristic is one that applies the same search method at each generation of the algorithm. A mixed strategy  metaheuristic is one that selects a search method probabilistically from a set of strategies at each generation. For example, a search strategy may be mutation  or crossover. Thus a classical genetic algorithm, that applies mutation with probability 0.9 and crossover with probability 0.1, belong to mixed strategy heuristics. A (1+1) evolutionary algorithm using mutation but no crossover is a  pure strategy metaheuristic.  Previously mixed strategy evolutionary programming,   integrating several mutation operators, has been designed for numerical optimization \citep{dong2007evolutionary}. Experimental results show that  mixed strategy evolutionary programming outperforms   pure strategy evolutionary programming with  a single mutation operator~\citep{shen2010mixed}.

The first goal of this paper is to conduct an empirical comparison of the performance between mixed strategy and pure strategy evolutionary algorithms (EAs for short) on the 0-1 knapsack problem.  Here the performance is measured by the best  fitness value found in 500 generations. In experiments, two novel mixed strategy EAs are proposed to solve the 0-1 knapsack problem. 

The second but more important goal  is to provide a theoretical answer to the question:   when do mixed strategy metaheuristics  outperform  pure strategy metaheuristics?  In theoretical analysis, the performance of a metaheuristic is measured by the expected number of total fitness evaluations to find an optimal solution (called the \emph{expected runtime}).  

Despite the popularity of hybrid metaheuristics in practice,   the theoretic work on   hybrid metaheuristics is very limited \citep{lehre2013runtime}.  One result is based on the asymptotic convergence rate  \citep{he2012pure}. The  asymptotic convergence rate is how fast an iteration algorithm converge to the solution per iteration \citep{varga2009matrix}.  It is proven in \citep{he2012pure} that any mixed strategy (1+1) EA (consisting of several mutation operators) performs no  worse than   the worst pure strategy  EA (using a single mutation operator). If mutation operators are mutually complementary, then it is possible to design a   mixed strategy (1+1) EA  better than  the  best pure strategy (1+1) EA.

Another result  is based on the runtime analysis of selection
hyper-heuristics \citep{lehre2013runtime}.  It  shows that mixing different
neighbourhood or move acceptance operators can be more
efficient than using stand-alone individual operators in some cases. But the discussion is restricted to simple (1+1) EAs for simple problems such as the OneMax and GapPath functions. 
 
This paper is different from our previous work \citep{he2012pure} in two points.  The expected runtime is employed to theoretically measure the performance of metaheuristics, while the asymptotic  convergence rate is used in \citep{he2012pure}.  The current paper discuss  population-based metaheuristics while \citep{he2012pure} only analysed (1+1) EAs.

The rest of this paper is organized as follows. Section \ref{secExperiment} gives experimental results that show  mixed strategy may outperform pure strategy. Section \ref{secTheory}  provides the sufficient and necessary condition when mixed strategy may outperform pure strategy  in general. Section \ref{secConclusions} concludes the   paper.

 
\section{Evidence from Experiment: Mixed Strategy May Outperform Pure Strategy}
\label{secExperiment}
This section conducts an empirical comparison of  the performance between  mixed strategy EAs and pure strategies EAs. A classical NP-hard problem, the 0-1 knapsack problem~\citep{martello1990knapsack}, is used in  the empirical study. 
\subsection{Evolutionary Algorithms for the 0-1 Knapsack Problem}
The 0-1 knapsack problem  is described as follows:
\begin{equation}
\begin{array}{lll}
 \mbox{maximize } & \sum^n_{i=1} v_i x_i, \nonumber\\
 \mbox{subject to }&\sum^n_{i=1} w_i x_i \le C, \nonumber
\end{array}
\end{equation}
where $v_i>0$ is the value of item $i$,  $w_i>0$ the weight of
item $i$, and $C>0$ the knapsack capacity. $$ x_i =
\left\{\begin{array}{llll} 1 & \mbox{if item $i$ is
selected in the knapsack}, \\
0 & \mbox{otherwise}.
\end{array}
\right.
$$
A solution  is represented by a vector (a binary string) $\vec{x}=(x_1, \cdots, x_n)$. If a solution $\vec{x}$ violates the constraint, then it is called infeasible. Otherwise it is called feasible.
 
There are several ways to handle the constrains in the knapsack problem \citep{michalewicz1996genetic}. The method  of  repairing infeasible solutions is used in the paper since it is more efficient than other methods \citep{he2007comparison}.  Its idea is simple:  if an infeasible solution is generated, then it will be
repaired to a feasible solution. The repairing
procedure is described  as follows:
 
\begin{algorithmic} 
\STATE \textbf{input} $\vec{x}$;
\IF{$\sum^n_{i=1} x_i w_i >C$}
\STATE $\vec{x}$ is infeasible;
\WHILE{($\vec{x}$ is infeasible)}
\STATE $i=:$ \textbf{select} an item from the knapsack;
\STATE set $ x_i=0$;
\IF{$\sum^n_{i=1} x_i w_i \le C$}
\STATE $\vec{x}$ is feasible;
\ENDIF
\ENDWHILE
\ENDIF
\STATE \textbf{output} $\vec{x}$.
\end{algorithmic}

There are different \textbf{select} methods in the repairing procedure. Two of them are described as follows.
\begin{enumerate}
\item
\emph{Random repair:}     select an item from the knapsack at random and remove it from the knapsack.

\item
\emph{Greedy repairing:}
sort all items  according to the order of the ratio $ {v_i}/{w_i}$,  then  select the   item with the smallest ratio and remove it from the knapsack.
\end{enumerate}

The fitness function is defined as
$$
f (\vec{x})=    \sum^n_{i=1} x_i v_i, \mbox{ if $\vec{x}$ is feasible,}  
$$
Thanks to the repairing method, no need to define the fitness for infeasible solutions.

A pure strategy EA  for solving the 0-1 knapsack problem  is described  as follows.
\begin{algorithmic} 
\STATE \textbf{input} a fitness function;
\STATE generation counter $ t\leftarrow 0$;
\STATE initialize population $ \Phi_0$;
\STATE an archive   keeps the best solution in $\Phi_0$; 
\WHILE{$t$ is less than a threshold}
\STATE children population $\Phi_{t+1/2}\leftarrow$   mutated from $\Phi_t$; \IF{a child is an infeasible solution}
\STATE then repair it into a feasible solution;
\ENDIF
\STATE new population $\Phi_{t+1}\leftarrow$   selected from $ \Phi_t, \Phi_{t+1/2}$;
\STATE  update the archive  if the best solution in $\Phi_{t+1}$ is better than it; 
\STATE $t\leftarrow t+1$;
\ENDWHILE \STATE \textbf{output} the maximum of the fitness function.
\end{algorithmic}

A mixed strategy EA  for solving the 0-1 knapsack problem  is  almost the same as the above algorithm, except one place: 
\begin{algorithmic} 
\STATE choose a mutation operator probabilistically;
\STATE children population $\Phi_{t+1/2}\leftarrow$    children mutated from $\Phi_t$. 
\end{algorithmic}
  The description of mutation operators is given in the next subsection.  The selection operator is the same in  pure strategy  and mixed strategy EAs. A  mixed strategy  then means a probability distribution of choosing  mutation operators.

\subsection{Pure Strategy and Mixed Strategy Evolutionary Algorithms}
Four pure strategy EAs are constructed based on four different mutation operators. 
The first  mutation operator is standard bitwise mutation. It is  independent on the 0-1 knapsack problem. The related  EA is  denoted by PSb.

\begin{itemize}
\item \emph{Bitwise Mutation.}  Flip each bit $x_i$ to $1-x_i$ with probability $\frac{1}{n}$. 
\end{itemize}

The second mutation operator is  problem-specific. It is based on   heuristic knowledge:  an item with a bigger value is more likely to appear in the knapsack. The  related EA is  denoted by PSv.

\begin{itemize}
\item \emph{Mutation based on values.} 
If a bit $x_i=0$, then flip it to $1$ with probability
\begin{align}
\frac{v_i}{\sum^n_{j=1} v_j}.
\end{align}
If a bit $x_i=1$, then flip it to $0$ with probability
\begin{align}
\frac{1/v_i}{\sum^n_{j=1} 1/v_j}.
\end{align}

 \end{itemize}

The third   mutation operator is based on  heuristic knowledge too: an item with a smaller weight is more likely to appear in the knapsack.  The corresponding EA is  denoted by PSw.

\begin{itemize}
\item \emph{Mutation based on weights.}  
If a bit $x_i=0$, then flip it to $1$ with probability
\begin{align}
\frac{1/w_i}{\sum^n_{j=1} 1/w_j}.
\end{align}
If a bit $x_i=1$, then flip it to $0$ with probability
\begin{align}
\frac{w_i}{\sum^n_{j=1} w_j}.
\end{align}

\end{itemize}

The fourth mutation operator is constructed from heuristics knowledge: 
first  calculate the ratio between the value and weight for each item:
\begin{align}
r_i =\frac{v_i}{w_i}.
\end{align}
Then an item with a bigger ratio  is more likely to appear in the knapsack. The related EA is  denoted by PSr.

\begin{itemize}
\item \emph{Mutation based on the ratio between value  and weight.} 
If a bit $x_i=0$, then flip it to $1$ with probability
\begin{align}
\frac{r_i}{\sum^n_{j=1} r_j}.
\end{align}
If a bit $x_i=1$, then flip it to $0$ with probability
\begin{align}
\frac{1/r_i}{\sum^n_{j=1} 1/r_j}.
\end{align}

\end{itemize}

Two novel mixed strategy EAs are designed in the experiments. One is to set  a fixed probability distribution of choosing  mutation operators for all generations.  The algorithm is called \emph{static}, denoted by MSs.

\begin{itemize}
\item \emph{statically mixed strategy:}    choose each mutation
operator based on a fixed probability, for example,  $(0.25, 0.25, 0.25, 0.25)$  for the  four pure strategies.
\end{itemize}

The other is to dynamically adjust the   probability distribution of choosing mutation operators. If a better solution is generated by applying a mutation operator this generation, then  the operator will be chosen with a higher probability. This kind of mixed strategy EAs  is called \emph{dynamic}, denoted by MSd.

\begin{itemize}
\item \emph{dynamically mixed strategy:}  The updating procedure of the mixed strategy is the same as that in \citep{dong2007evolutionary}.   
For each individual    in
population $\Phi_{t+1}$, update its mixed strategy as follows.  If the   individual's parent generates a child  via mutation PS and the child is selected into  population $\Phi_{t+1}$, then assign the probabilities of choosing mutation PS and other mutation PS' to be 
$$
\begin{array}{llll}
 P_{t+1}(PS)= P_{t}(PS) +\frac{1-P_{t}(PS)}{4},   \\
  P_{t+1}(PS')= P_t (PS') - \frac{P_{t}(PS')}{4}, \quad 
PS' \neq PS,
\end{array}
$$

Otherwise assign 
$$
\begin{array}{llll}
P_{t+1}(PS)= P_{t}(PS)  - \frac{P_{t}(PS)}{4},   \\
P_{t+1}(PS')= P_t (PS') + \frac{1- P_{t}(PS')}{4}, \quad 
PS' \neq PS,
\end{array}
$$
 
\end{itemize}

\subsection{Experiments}
Experiments are conducted on different types of instances of the 0-1 knapsack problem.
According to the correlation between values and weights,  the instances of the problem are classified into three types~\citep{martello1990knapsack,michalewicz1996genetic}: given two positive parameters $A $ and $B$, 
\begin{enumerate}
\item {\it uncorrelated knapsack}: $v_i$ and $w_i$ uniformly random in $[1, A ]$;

\item {\it weakly correlated knapsack}: $w_i$ uniformly random in $[1,
A ]$; and $v_i$ uniformly random in $[w_i-B, w_i+B]$ (if for some
$j$,  $v_i \le 0$, then the random generation procedure should be
repeated until $v_i>0$);

\item {\it strongly correlated knapsack}: $w_i$ uniformly random in $[1,
A ]$; and $v_i= w_i+B$;
\end{enumerate}

In the experiments, $A $ and $r$ are set to be
$
A  =\frac{n}{20}$ and $B=\frac{n}{20}$.

Based on the capacity,  the instances of the knapsack problem are  classified into two
types~\citep{martello1990knapsack,michalewicz1996genetic}:
\begin{enumerate}
\item
{\it restrictive capacity knapsack:} the knapsack capacity  is
small, where $  C =2A .$ 

\item
{\it average capacity knapsack:}  the knapsack  capacity   is large,
where $ C =0.5 \sum^n_{i=1} w_i.$ 
\end{enumerate}

Hence we will compare two mixed strategy EAs and four pure strategy EAs on six different types of instances below:
\begin{enumerate}
\item uncorrelated  and restrictive capacity knapsack,
\item  weakly correlated and restrictive capacity knapsack,
\item  strongly correlated  and restrictive capacity knapsack,
\item uncorrelated  correlated  and average capacity knapsack,
\item  weakly correlated  and average capacity knapsack,
\item  strongly correlated  and restrictive average capacity knapsack.
\end{enumerate}

Furthermore the experiments are split into two groups based on the repairing method: (1)    greedy repair, (2)   random repair.

The experiment setting is described as follows. For each type of the 0-1 knapsack problem, three instances with  100, 250 and 500 items are generated at random.
The population size is set to 10. The maximum of generations is  500. The initial population is chosen at random.  

Tables \ref{result1} and \ref{result2}  give the experimental results. The number  in the table is the best fitness values found in 500 generations. It is averaged over 10 independent runs for each instance.

   \begin{table}[ht]
 \begin{center}
  \caption{greedy repair: the best fitness value found in 500 generations, averaged over 10 runs for each instance}
      \label{result1}
 \begin{tabular}{ccccccc}\hline 
\multicolumn{7}{c}{uncorrelated and restrictive capacity knapsacks }\\   
 \hline  
 $n$ &   MSs &   MSd &  PSb  &  PSv  &  PSw  &  PSr\\ 
  \hline 
 100 & \textit{285}& \textbf{300}  & 279  & 281  & 283 & 277 \\ 
 250 & \textit{1609} & \textbf{1655}  & 1601 & 1539 & 1528 & 1513\\ 
500 & 5625 & 5703  & 5515  & \textbf{5794} & 5140  & 5504 \\ 
   \hline 
\multicolumn{7}{c}{weakly correlated and restrictive capacity knapsacks  }\\    
 \hline 
 $n$ &  MSs & MSd &  PSb  &  PSv  &  PSw  &  PSr\\ 
 \hline
 100 & \textit{342} & \textbf{353}  & 331  & 289 & 306  & 325\\ 
 250 & 1651  & \textbf{1695}  & 1583 & 1514 & 1668 & 1650\\ 
 500 & \textit{5319}  & \textbf{5545}  & 5161 & 4595  & 4810 & 4710\\ 
 \hline
\multicolumn{7}{c}{strongly correlated  and restrictive capacity knapsacks  }\\    
 \hline 
 $n$ &  MSs & MSd &  PSb  &  PSv  &  PSw  &  PSr\\ 
  \hline 
 100 & 671  & \textbf{683} & 678 & 662  & 655 & 665 \\ 
 250 & 4126  & 4261 & 4212 & 4170 & \textbf{5023}  & 3980\\ 
 500 & 15273 & \textbf{15537}  & 14959  & 15226  & 15367 & 14179 \\ 
 \hline 
\multicolumn{7}{c}{uncorrelated and average capacity knapsacks  }\\   
 \hline 
 $n$ &  MSs & MSd &  PSb  &  PSv  &  PSw  &  PSr\\ 
  \hline 
 100 &  295   & \textbf{299}  & 293 & 292 & 288 & 295\\ 
 250 & 1616 & \textbf{1650}  & 1616  & 1619 & 1583 & 1609 \\ 
 500 & 5751 & 5958  & 5670 & \textbf{5963}  & 5601  & 5663\\ 
  \hline 
\multicolumn{7}{c}{weakly correlated and  average capacity knapsacks  }\\  
 \hline  
 $n$ &  MSs & MSd &  PSb  &  PSv  &  PSw  &  PSr\\ 
  \hline 
 100 & \textit{387}  & \textbf{395} & 355 & 344  & 362  & 349\\ 
 250 & 1976  & \textbf{2014} & 2009 & 1924 & 1997  & 1956\\ 
 500 & \textit{7284} & \textbf{7505} & 6839 & 6966  & 7048 & 7049\\ 
  \hline
\multicolumn{7}{c}{strongly correlated  and  average capacity knapsacks  }\\   
 \hline 
 $n$ &  MSs & MSd &  PSb  &  PSv  &  PSw  &  PSr\\ 
  \hline 
 100 & 716 & \textbf{730} & 718  & 721  & 714 & 716 \\ 
 250 & \textit{4372} & \textbf{4409} & 4301 & 4220 & 4135  & 4218\\ 
 500 & 15525  & \textbf{15868}  & 15746 & 15819 & 15552  & 14828\\ 
\hline
 \end{tabular}
 \end{center}
 \end{table}

     \begin{table}[ht]
 \begin{center}
  \caption{random repair: the best fitness value found in 500 generations, averaged over 10 runs for each instance}
  \label{result2}
 \begin{tabular}{ccccccc}\hline 
\multicolumn{7}{c}{uncorrelated and restrictive capacity knapsacks  }\\   
 \hline  
 $n$ &   MSs &   MSd &  PSb  &  PSv  &  PSw  &  PSr\\ 
  \hline 
100 &  \textit{167}  & \textbf{170} & 161 & 160 & 155 & 166\\ 
250 & 850 & \textbf{876} & 852 & 842 & 810  & 846\\ 
500 & \textit{2550} & \textbf{2675} & 2440  & 2513 3 & 2496 & 2426\\ 
   \hline 
\multicolumn{7}{c}{weakly correlated and restrictive capacity knapsacks  }\\  
 \hline 
 $n$ &  MSs & MSd &  PSb  &  PSv  &  PSw  &  PSr\\ 
 \hline
100 &  \textit{236} & \textbf{242} & 230 & 229 & 226 & 222 \\ 
250 & 1066  & \textbf{1134}  & 1046 & 1058  & 1098  & 1067 \\ 
500 & \textit{3947}  & \textbf{4071} & 3719  & 3741 & 3713 & 3815\\ 
   \hline
\multicolumn{7}{c}{strongly correlated  and restrictive capacity knapsacks   }\\    
 \hline 
 $n$ &  MSs & MSd &  PSb  &  PSv  &  PSw  &  PSr\\ 
 \hline
 100 & 405 & \textbf{416 }& 403 &  405 & 389& 408 \\ 
  250 & 2171 & 2204 & 2188 & \textbf{2273} & 2138 & 2205 \\ 
500 & \textit{7028} &	\textbf{7078} &	6981 &	6883 &	6958 &	6946  \\ 
 \hline
\multicolumn{7}{c}{uncorrelated   and average capacity knapsacks }\\     
 \hline 
 $n$ &  MSs & MSd &  PSb  &  PSv  &  PSw  &  PSr\\ 
  \hline 
100 & 225 & \textbf{234} & 218 & 231 & 212  & 231 \\ 
250 & 1236  & \textbf{1266} & 1197  & 1208  & 1070  & 1263\\ 
500 & 4669 & 4697  & 4443 & 4674 & 3922  & \textbf{4716} \\ 
   \hline 
\multicolumn{7}{c}{weakly correlated   and average capacity knapsacks }\\    
 \hline  
 $n$ &  MSs & MSd &  PSb  &  PSv  &  PSw  &  PSr\\ 
  \hline 
100& 295  & \textbf{311} & 290 & 294 & 292 & 304 \\ 
250& \textit{1520}  & \textbf{1530} & 1497 & 1491 & 1374  & 1519\\ 
500& 5641 & \textbf{5769} & 5300  & 5525  & 4999 & 5710\\ 
    \hline
\multicolumn{7}{c}{strongly correlated   and average capacity knapsacks   }\\    
 \hline 
 $n$ &  MSs & MSd &  PSb  &  PSv  &  PSw  &  PSr\\ 
  \hline 
 100 &476 & \textbf{493} & 479 & 483& 470 & \textbf{493}\\ 
250 &2650 & 2716 & 2613 & 2610  & 2586  & \textbf{2721} \\ 
500 &10156 & 10285  & 10119 & 10131  & 10065 & \textbf{10393}\\ 
\hline
 \end{tabular}
 \end{center}
 \end{table}
 
  Following a simple calculation, we see that  the dynamically mixed strategy EA, MSd,  is the best in 77.8\% instances   and   equally the best in 2.8\% instances.  If we compare the statically mixed strategy EA with the four pure strategies, then we see that MSs is better than the four pure strategy EAS in 36.1\% instances (marked in italic type in the tables).  
  
Experimental results show  mixed strategy EAs outperform pure strategy EAs in up tp 77.8\% instances, but not always. Naturally it raises the question: under what condition, a mixed strategy EA may  outperform a pure strategy EA. This question is seldom answered rigorously before.  
  
\section{Support of Theory: Mixed Strategy May Outperform Pure Strategy}
In this section, we conduct a theoretical comparison of the performance between mixed strategy metaheuristics and pure strategy metaheuristics.  

\label{secTheory}
\subsection{Meta-heuristics and Markov Chains} 
Without lose of generality,  consider the  problem of maximising a fitness function:
\begin{equation}
   \mbox{maximize }   f(x),    
\end{equation}
where $x$ is a variable and   its definition domain is a finite set.  
 
The metaheuristics considered in the paper are formalised as Markov chains. Initially construct a population  of solutions $\Phi_0$;  from   $\Phi_0$, then generate a new population  of solutions $\Phi_1$;  from $\Phi_1$, then  generate a new population of solutions $\Phi_2$, and so on. This procedure is repeated until a stopping condition is satisfied. A sequence of populations is then generated
$$\Phi_0 \to \Phi_1 \to \Phi_2 \to \cdots.$$
An archive  is used for recording the best found solution so far.   The archive is not involved in  generating  a new  population.  In this way, the best found solution is preserved for ever (called \emph{elitist}). 
The metaheuristics algorithm with an archive  is described below. In the algorithm,   the number of fitness evaluations for each generation is invariant. 
 
\begin{algorithmic} 
\STATE set counter  $t$ to 0; 
\STATE initialize a population  $ \Phi_{0}$;
\STATE an archive   keeps the best solution in $\Phi_0$; 
\WHILE{the archive is not an optimal solution}
\STATE a new population  $\Phi_{t+1}$ is generated from $\Phi_t$;
\STATE  update the archive  if the best solution in $\Phi_{t+1}$ is better than it; 
\STATE  counter  $t$ is increased by $1$;
\ENDWHILE
\end{algorithmic}
 
The procedure of generating $\Phi_{t+1}$ from $\Phi_t$ can be  represented by  transition probabilities among populations:  
\begin{align}
P(X,Y):=P(\Phi_{t+1}= Y \mid  \Phi_t = X),\end{align}
where populations $\Phi_t, \Phi_{t+1}$ are variables and $X,Y$ are their values (also called states).  The transition probabilities $P(X,Y)$ form the transition matrix of a Markov chain, denoted by $\mathbf{P}$. 

\begin{definition} If a  transition matrix $\mathbf{P}$ for generating new populations is independent of $t$, then it is called a \emph{pure strategy}.
A \emph{mixed strategy}  is a   probability distribution of choosing a  pure strategy from a set of strategies.
\end{definition}

In theory, the stopping criterion is that  the algorithm halts once an optimal solution is found. It is taken for the convenience of analysing  the first time of finding an optimal solution (called \emph{hitting time}). If $\Phi_t $ includes an optimal solution, then assign
$$\Phi_t=    \Phi_{t+1}= \Phi_{t+2} =\cdots$$    for ever. As a result, the population   sequence $\{\Phi_t\}$   is formulated by a homogeneous Markov chain \citep{he2003towards}.

Since a state in the optimal set  is always absorbing, so the  transition matrix $\mathbf{P} $  can be written in the canonical form,
\begin{equation}
\label{equCanonicalForm} 
\mathbf{P} =  \begin{pmatrix} 
  \mathbf{I} & \mathbf{{O}} \\
   \mathbf{R}  & \mathbf{Q} 
 \end{pmatrix},
\end{equation}
where  
$\mathbf{I}$ is a unit matrix,   $\mathbf{O}$ a zero matrix and $ \mathbf{Q}$  a matrix for    transition probabilities among non-optimal  populations.
  $\mathbf{R}$ denotes the transition  probabilities from non-optimal populations to  optimal populations.

Let   $m(X)$ denote the expected number of generations needed to find an optimal solution when  $\Phi_0$ is at state  $X$ for the first time (thereafter it is abbreviated by the \emph{expected hitting time}). Clearly for any initial population $X$ in the optimal set,    $m(X)$ is $0$. 
Let   $( X_1, X_2, \cdots   )$ represent all populations in the non-optimal set and   the vector ${\vec{m}}$   denote    their expected hitting times  respectively
$$\vec{m}  = (m(X_1), m(X_2), \cdots   )^T.$$

Since the number of fitness evaluations for each generation is invariant,  the total number of fitness evaluations (called \emph{runtime}) equals to the expected hitting time $\times$ the number of fitness evaluations of a generation. 

The following theorem  \citep[Theorem 11.5]{grinstead1997introduction} shows that the expected hitting time can be calculated from the transition matrix. 
\begin{theorem} [Fundamental Matrix Theorem] 
\label{lem1} 
The expected hitting time  is given by
\begin{align}
\vec{m}=(\mathbf{I-Q})^{-1}\vec{1},
\end{align}
where $\vec{1}$ is a  vector all of whose entries are $1$,  the matrix $\mathbf{N}=(\mathbf{I}-\mathbf{Q})^{-1}$ is called the \emph{fundamental matrix}.
\end{theorem}

Two special values of the expected hitting time are often used to evaluate the performance of metaheuristics. 
The first value is the average of the expected hitting time,  given by 
\begin{align}
\bar{m}= \frac{1}{\mid \mathcal{S} \mid} \sum_{X \in \mathcal{S}} m(X). 
\end{align}
where $\mathcal{S}$ denotes the set of all populations. The average corresponds to the case when the initial population is chosen at random.

The second value is the maximum of the expected hitting time, given by
\begin{align}
 \max \{m(X); X \in \mathcal{S} \},
\end{align} 
The maximum  corresponds to the case when the initial population is chosen at the worst state.

The population set $\mathcal{S}$ is divided into two parts:   $\mathcal{S}_{\non}$ denotes  the set of all populations which don't include any optimal solution and $\mathcal{S}_{\opt}$  the set of all populations which include at least one  optimal solution.

\subsection{Drift Analysis} 
Drift analysis is used  for bounding the expected hitting time  of metaheuristics~\citep{he2001drift}.
In drift analysis, a \emph{distance function} $d(X)$ is a non-negative function. 
Let  $(X_1, X_2, \cdots)$   represent  all populations in the non-optimal set and  $\vec{d}$ denote  the vector
$$(d(X_1), d(X_2), \cdots)^T.$$ 

\begin{definition} Let $\mathbf{P}$ be the Markov chain associated with ametaheuristic and $d(X)$ be a distance function. For a  non-optimal population $X$, the \emph{drift}   at state $X$   is  
$$
\Delta (X):=d(X)-\sum_{Y \in \mathcal{S}_{\non}}  d(Y)  P(X,Y).
$$
\end{definition}
The drift represents the one-step progress rate towards the global optima. 
Since the Markov chain $\{\Phi_t; t=0,1, \cdots \}$ is homogeneous, the above  drift is independent of $t$.

The following   theorem  is a variant of the original drift theorem \citep[Theorems 3 and 4]{he2003towards}. 

\begin{theorem} [Drift Analysis Theorem]
(1) If the drift satisfies that $\Delta  d(X) \ge 1$ for any state $X$, and   $\Delta d(X) > 1$ for some state $X$, then    the expected hitting time satisfies that  $m  (X) \le d(X) $ for any initial population $X$,  and    $m (X) < d(X)$ for   some initial population $X$.

(2) If the drift satisfies that $\Delta  d(X) \le 1$ for any state $X$, and   $\Delta d(X) < 1$ for some state $X$, then    the expected hitting time satisfies that  $m  (X) \ge d(X) $ for any initial population $X$,  and    $m (X) > d(X)$ for   some initial population $X$.

\end{theorem}
 
\begin{proof}
We only prove the first conclusion. The second conclusion can be proven in a similar way. 

The notation $\succ$ is introduced in the proof as follows: given two vectors $\vec{a} =[a_{i}]$ and $\vec{b} =[b_{i}]$, if for all $i$, $a_{i} \ge b_{i}$ and for some $i$, $a_{i} > b_{i}$, then write it by $\vec{a}  \succ \vec{b} $. Similarly given two matrices $\mathbf{A} =[a_{ij}]$ and
$\mathbf{B} =[b_{ij}]$, if for all $i,j$, $a_{ij} \ge b_{ij}$ and for some pair $i,j$, $a_{ij} > b_{ij}$, then   write it by
$\mathbf{A}  \succ \mathbf{B} $.

  Let $\vec{1}$ denote the vector whose entries are $1$, $\vec{0}$  the vector whose entries are $0$ and $\mathbf{O}$ a matrix whose entries are $0$. 
The condition of the theorem can be rewritten in an equivalent vector form:
\begin{align*}
 \vec{d} -\mathbf{Q} \vec{d} = \vec{1} +\vec{e},
\end{align*} 
where $\vec{e} \succ \vec{0}$.

Then we have 
\begin{align*}
 &\vec{d} -\mathbf{Q} \vec{d} - \vec{1} - \vec{e}=\vec{0},\\
&(\mathbf{I} - \mathbf{Q})^{-1}( \vec{d} -\mathbf{Q} \vec{d} - \vec{1} -\vec{e}) = \vec{0},\\
&(\mathbf{I} - \mathbf{Q})^{-1}( \vec{d} -\mathbf{Q} \vec{d} - \vec{1})= (\mathbf{I} - \mathbf{Q})^{-1}\vec{e} .
\end{align*}

Now let's bound the right-hand side. Since $\vec{e} \succ \vec{0}$, so   entry $e_j>0$ for some $j$.   ${\bf P}$ is a transition matrix, $\mathbf{Q}\succ \mathbf{O} $ and the spectral radius of
$\mathbf{Q}$ are less than $1$, so $\mathbf{N}=(\mathbf{I} -\mathbf{Q})^{-1} \succ \mathbf{O}$. Since no eigenvalue of
$\mathbf{N}$ is $0$,  for the $j$-column  of $\mathbf{N}$, at least one entry is greater than $0$ (otherwise $0$ will be an eigenvalue of $\mathbf{N}$).  Thus entry $n_{ij}>0$ for some $i$. Then $n_{ij} e_j >0$ and  
\begin{align*}
(\mathbf{I} - \mathbf{Q})^{-1}\vec{e} \succ \vec{0}.
\end{align*}

Hence we get
 \begin{align*}
&(\mathbf{I} - \mathbf{Q})^{-1}( \vec{d} -\mathbf{Q} \vec{d} - \vec{1}) \succ \vec{0},\\
& \vec{d}\succ  (\mathbf{I} -\mathbf{Q})^{-1} \vec{1}.
\end{align*}

From the Foundational Matrix Theorem, we know that $$(\mathbf{I-Q})^{-1}\vec{1}=\vec{m}.$$ Then we get   
$\vec{d} \succ \vec{m} .
$  
This inequality implies the conclusion of the theorem.
\end{proof}  

The following consequence is directly derived from  the Fundamental Matrix Theorem.

\begin{corollary} 
\label{corDrift1} Let the distance function $d(X)=m(X)$, then the drift satisfies
$\Delta(X)=1$ for any state $X$ in the non-optimal set.
\end{corollary}

\begin{proof}
From the Fundamental Matrix Theorem: $(\mathbf{I-Q})\vec{m}=\vec{1}$. Then  we write it in the entry form and it gives $\Delta(X)=1$ for any state $X$ in the non-optimal set.
\end{proof}

\subsection{One Pure Strategy is Inferior or Equivalent to another Pure Strategy}
In the subsection, we investigate the case that it is impossible to design a mixed strategy better than a pure strategy. 
Consider two metaheuristics: one using a pure strategy PS1 (PS1 for short) and another  using a pure strategy PS2 (PS2 for short). Let $\vec{m}_{PS1}$ be the vector representing the expected hitting times with respect to PS1 and the distance function $d(X)=m_{PS1} (X)$. For PS1, denote its corresponding drift at state $X$ by $\Delta_{PS1}(X)$: 
\begin{align*}
\Delta_{PS1} (X) =d(X)-\sum_{Y \in S_{\non}}P_{PS1} (X,Y) d(Y), 
\end{align*}
where $P_{PS1} (X,Y)$ represents the transition probability from $x$ to $Y$. 
According to Corollary~\ref{corDrift1}, the drift  $\Delta_{PS1}(X)=1$ for all $X$ in the non-optimal set.

For PS2, denote the corresponding drift at state $x$ by $\Delta_{PS1}(X)$: 
\begin{align*}
\Delta_{PS2} (X) =d(X)-\sum_{Y \in S_{\non}}P_{PS2} (X,Y) d(Y).
\end{align*}

First we propose the ``inferior'' condition. 
\begin{definition}
If the drift of PS1 and  that of PS2 satisfy  $\Delta_{PS1} (X) \ge \Delta_{PS2} (X)$ for any state $X$, and  $\Delta_{PS1} (X) > \Delta_{PS2} (X)$ for some state $X$, then we call PS2 is \emph{inferior} to  PS1.  
\end{definition}

We consider the   mixed strategy metaheuristic derived from PS1 and PS2  (MS for short)   at the population level: the probability of choosing a search strategy is the same for all individuals.  Suppose  population $\Phi_t$ is at state $X$,   we denote the probability of choosing PS1 by $P_X(PS2)$ and the probability of choosing PS1 by $P_X(PS2)$.  The sum $P_X(PS1) +P_X(PS2)=1.$

\begin{lemma}
\label{lemInferior}
If   PS2 is inferior to   PS1, then for any mixed strategy metaheuristics derived from PS1 and PS2, the expected hitting time of MS satisfies that 
  $m_{MS} (X) \ge m_{PS1} (X)$  for any initial population $X$,  $m_{MS} (X) > m_{PS1} (X)$ for some state $X$.
\end{lemma}

\begin{proof}
Let $\Delta_{MS} (X)$ denote the drift associated with MS.
For any state $X$, the drift of MS is 
\begin{align*}
\Delta_{MS} (X)=& d(X)-\sum_{Y \in S_{\non}}P_{MS} (X,Y) d(Y)\\
=& P_X(PS1)  [d(X)-\sum_{Y \in S_{\non}}   P_{PS1} (X,Y) d(Y) ] \\
 +&P_X(PS2)  [ d(X)-\sum_{Y \in S_{\non}}   P_{PS2} (X,Y) d(Y)] \\
=& P_X(PS1)  \Delta_{PS1} (X)+P_X(PS2) \Delta_{PS2} (X).
\end{align*}

Since PS2 is inferior to PS1,   we know that   $\Delta_{PS1} (X) \ge \Delta_{PS2} (X)$ for any state $X$,  and    $\Delta_{PS1} (X) > \Delta_{PS2} (X)$ for some $X$.
Therefore    $\Delta_{PS1} (X)=1 \ge \Delta_{MS} (X)$ for any state $X$,  and    $\Delta_{PS1} (X)=1 > \Delta_{MS} (X)$ for some state $X$.

Applying the Drift Analysis Theorem, we get the  conclusion:  the expected hitting time satisfies that $m_{MS} (X) \ge m_{PS1} (X)$ for any initial population $X$, 
and    $m_{MS} (X) > m_{PS1} (X)$ for some initial population $X$.
\end{proof}

From the above lemma, we  infer two corollaries.
\begin{corollary}
\label{cor2}
If   PS2  is inferior to PS1, then for any   mixed strategy MS derived from PS1 and PS2,  its average of the expected hitting time  is greater than that of PS1.
\end{corollary}

\begin{proof}
According to the above lemma,   the expected hitting time satisfies that $m_{MS} (X) \ge m_{PS1} (X)$ for any initial population $X$, 
and    $m_{MS} (X) > m_{PS1} (X)$ for some initial population $X$. From  the definition of  average,  
\begin{align*}
\bar{m} =\frac{1}{\mid \mathcal{S} \mid} \sum_{X \in \mathcal{S}} m(X),
\end{align*} then we get $\bar{m}_{MS} > \bar{m}_{PS1}$.
\end{proof}

\begin{corollary}
\label{cor3}
If   PS2  is inferior to PS1, then for any   mixed strategy MS derived from PS1 and PS2,  its maximum of the expected hitting time  is not less than that of PS1.
\end{corollary}

\begin{proof}
According to the above lemma,  the expected hitting time satisfies that $m_{MS} (X) \ge m_{PS1}(X)$ for any initial population $X$. Then we get $$\max\{ m_{MS}(X); X \in \mathcal{S}\} \ge \max\{ m_{PS1}(X); X \in \mathcal{S}\}$$ and prove  the conclusion.
\end{proof}

Next we propose the ``equivalent'' condition.
\begin{definition}
If the drift of PS1 and  that of PS2 satisfy  $\Delta_{PS1} (X) = \Delta_{PS2} (X)$ for any state $X$, then we call PS1 is \emph{equivalent} to  PS2. 
\end{definition}
 
The following lemma is direct corollary of the Drift Analysis Theorem.

 \begin{lemma}
\label{lemEquivalent}
If    PS2 is equivalent to  PS1, then for any mixed strategy MS derived from PS1 and PS2,   its the expected hitting time satisfies that  $m_{MS} (X) = m_{PS1} (X)$ for any initial population $X$.
\end{lemma}

 \subsection{One Pure Strategy is Complementary to Another Pure Strategy}
In the subsection, we investigate the case that it is possible to design a mixed strategy better than a pure strategy.   We propose the ``complementary'' condition. Like the previous subsection, the distance function $d(X)=m_{PS1} (X)$.
 
\begin{definition}  
If the drift of PS1 and  that of PS2 satisfy  $\Delta_{PS1} (X)<\Delta_{PS2} (X)$ for some state $X$,   then we call PS2 is \emph{complementary} to PS1.
\end{definition}

\begin{lemma}
\label{lemComplementary}
If    PS2 is complementary to  PS1, then there exists   a mixed strategy MS derived from PS1 and PS2,  and its the expected hitting time satisfies that  $m_{MS} (X) \le m_{PS1} (X)$ for any initial population $X$, and    $m_{MS} (X) < m_{PS1} (X)$ for some initial population $X$.
\end{lemma}

\begin{proof}
First we construct a   mixed strategy   derived from PS1 and PS2. The  construction follows a well-known principle: at one state, if a pure strategy has a better performance than the other at a state, then the strategy should be applied with a higher probability at that state. 
\begin{enumerate}
\item When   $\Phi_t$ is at   state $X$, if the drift $\Delta_{PS1} (X)$ is greater than the drift $\Delta_{PS2} (X)$, then the  probability of choosing PS1 is set to 1, that is, $P_X(PS1) =1$.

\item When   $\Phi_t$ is at  state $X$,  if the drift $\Delta_{PS1} (X)$ equals to the drift $\Delta_{PS2} (X)$, then the  probability of choosing PS1 is set to any value between $[0,1]$, that is, $0\le P_X(PS1) \le 1$.

\item Since PS2 is complementary to PS1, so there exists   one state $X$  such that the drift $\Delta_{PS2} (X)$ is larger than the drift $\Delta_{PS1} (X)$.  When   $\Phi_t$ is at such a state $X$,   then the  probability of choosing PS2 is set to any value greater than $0$, that is, $0<P_X(PS2) \le 1$.
\end{enumerate}

In this way   a mixed strategy  MS is constructed from PS1 and PS2.  

Next we bound the drift of the mixed strategy. For any state $X$, the drift of the mixed strategy is  
\begin{align*}
\Delta_{MS} (X)=& d(X)-\sum_{Y \in S_{\non}}P_{MS} (X,Y) d(Y)\\
=& P_X(PS1)  [d(X)-\sum_{Y \in S_{\non}}   P_{PS1} (X,Y) d(Y) ] \\
+&P_X(PS2)  [ d(X)-\sum_{Y \in S_{\non}}   P_{PS2} (X,Y) d(Y)] \\
=& P_X(PS1)  \Delta_{PS1} (X)+P_X(PS2) \Delta_{PS2} (X).
\end{align*}

Based on the construction of the mixed strategy, the analysis of the drift is classified into three cases. 
\begin{enumerate}
\item  \emph{$\Delta_{PS1} (X) > \Delta_{PS2} (X)$:}  in this  case, the probability of choosing PS1 is 1, that is, $P_X(PS1) =1$. Thus the drift satisfies: $\Delta_{MS} (X) =\Delta_{PS1} (X). $
 
\item \emph{$\Delta_{PS1} (X) = \Delta_{PS2} (X)$:} in this case, the drift satisfies: $\Delta_{MS} (X) =\Delta_{PS1} (X). $

\item \emph{$\Delta_{PS1} (X) < \Delta_{PS2} (X)$:} in this case,  the probability of choosing PS2 is greater than 0, that is, $P_X(PS2) >0$.
 Thus the drift satisfies: $\Delta_{MS} (X) <\Delta_{PS1} (X). $
\end{enumerate}

Summarising all three cases, we see that  the drift of the mixed strategy satisfies: $\Delta_{MS} (X) \ge \Delta_{PS1} (X)=1$ for any state $X$, and    $\Delta_{MS} (X) > \Delta_{PS1} (X)=1$ for some state $X$.

Finally applying the Drift Analysis Theorem, we come to the conclusion:  the expected hitting time satisfies:  $m_{MS} (X) \le m_{PS1} (X)$  for any initial population $X$,
 and    $m_{MS} (X) < m_{PS1} (X)$  for some initial population $X$.
\end{proof}

From the above lemma, we  draw two results about the average and maximum of the expected hitting times.

\begin{corollary}\label{cor4}
If   PS2 is complementary to   PS1, then there exists a mixed strategy MS derived from PS1 and PS2 and its average of the expected hitting time  is less than that PS1.
\end{corollary}

\begin{proof}
According to the above lemma, the expected hitting time satisfies:  $m_{MS} (X) \le m_{PS1} (X)$  for any initial population $X$,
 and    $m_{MS} (X) < m_{PS1} (X)$  for some initial population $X$.
From the definition
\begin{align*}
\bar{m} =\frac{1}{\mid \mathcal{S} \mid} \sum_{X \in \mathcal{S}} m(X),
\end{align*} then we get $\bar{m}_{MS}< \bar{m}_{PS1}$.
\end{proof}

\begin{corollary}
If   PS2 is complementary to   PS1, then there exists a mixed strategy MS derived from PS1 and PS2 and its maximum of the expected hitting time   is no more than that PS1.
\end{corollary}

\begin{proof}
According to the above lemma,  the expected hitting time satisfies:  $m_{MS} (X) \le m_{PS1} (X)$  for any initial population $X$.
Then we get $$\max\{ m_{MS}(X); X \in \mathcal{S}\} \le \max\{ m_{PS1}(X); X \in \mathcal{S}\}$$ and prove  the conclusion.
\end{proof}

\subsection{Complementary  Strategy Theorem}
Combining Lemmas~\ref{lemInferior},   \ref{lemEquivalent} and \ref{lemComplementary} together, we obtain our main result about mixed strategy metaheuristics. It gives an answer to the  question: under what condition,  mixed strategy metaheuristics may outperform  pure strategy metaheuristics.  
\begin{theorem}[Complementary  Strategy Theorem]
Consider two metaheuristics: one using pure strategy PS1   and another  using  pure strategy PS2.  The condition of PS2 being complementary to PS1 is sufficient and necessary if there exists a  mixed strategy MS derived from PS1 and PS2 such that:  $m_{MS} (X) \le m_{PS1} (X)$  for any initial population $X$,
 and    $m_{MS} (X) < m_{PS1} (X)$  for some initial population $X$.
 
 Furthermore  the condition of PS2 being complementary to PS1 is sufficient and necessary if there exists a  mixed strategy MS derived from PS1 and PS2 such that: the expected runtime of MS is no more than that of PS1 for any initial population $X$,
 and less than that of PS1 for some initial population $X$.
\end{theorem}

\begin{proof}
Given PS1 and PS2, their relation is classified into exact three exclusive types: PS2 is inferior, or equivalent, or complementary to PS1. Thus combining Lemmas~\ref{lemInferior},   \ref{lemEquivalent} and \ref{lemComplementary} together, we get the desired first conclusion.

Since  the expected runtime equals to the expected hitting time $\times$ the number of fitness evaluations of a generation, we obtain the second conclusion.
\end{proof}

The theorem can be explained intuitively as follows. 

\begin{enumerate}
\item If one pure strategy is inferior to another pure strategy, then it is impossible to design a mixed strategy with a better performance. So mixed strategy metaheuristics doesn't always outperform pure strategy metaheuristics. 

\item If one pure strategy is complementary to another one, then it possible to design  a  mixed strategy  better than the pure strategy. But it does not mean all mixed strategies will outperform the pure strategy. 

\item The construction of a better mixed strategy metaheuristics   should follow a general principle:  if using a pure strategy has a better progress rate (in terms of the drift) than that using the other at a state, then the strategy should be applied with a higher probability at that state. This principle is   general, but the design of a better mixed strategy is strongly dependent on the problem. 
\end{enumerate} 

For the average of the expected hitting time, we may obtain a similar consequence after combining Corollaries \ref{cor2}, \ref{cor4} and Lemma \ref{lemEquivalent} together.
\begin{corollary} 
 The condition of PS2 being complementary to PS1 is sufficient and necessary if   there exists a mixed strategy MS derived from PS1 and PS2 and its average of the expected hitting time   is less than than that of PS1.
\end{corollary}

But the sufficient and necessary condition for the maximum of the expected hitting time is more complex.

\subsection{An Example}
Consider an instance of the 0-1 knapsack problem: the value of items  $v_1=n$ and $v_i=1$  for $i=2, \cdots, n$,   the weight of
items   $w_1=n$ and $w_i=1$  for $i=2, \cdots, n$. The capacity   $C=n$.
The fitness function is  
\begin{equation}
f(x)= 
\left\{
\begin{array}{lll}
n, &\mbox{if } s_1=1, s_2 =\cdots s_n=0,\\
\sum^n_{i=1} s_i, &\mbox{if }  s_1=0,\\
\mbox{infeasible }, &\mbox{otherwise}.
\end{array}
\right.
\end{equation}

For the four pure EAs described in the previous section, it is easy to verify that
\begin{enumerate}
\item PSr is equivalent to PSb,

\item PSw is inferior to PSb,

\item  PSv is complementary to PSb.
\end{enumerate}   
 
Applying the Completerary Strategy Theorem, we know that 
\begin{enumerate}
\item combining  PSr with  PSb will not shorten  the expected runtime;
\item  combining PSw will PSb will not shorten the expected runtime too;

\item but combining PSv with PSb may reduce the expected  runtime.
\end{enumerate}   

\section{Conclusions}
\label{secConclusions}
The main contribution of the paper is the Complementary Strategy Theorem. From the theoretical viewpoint, the theorem provides an answer to the  question: under what condition,  mixed strategy metaheuristics may outperform  pure strategy metaheuristics.  The theorem asserts that given two metaheuristics where one uses a pure strategy PS1   and the other  uses a pure strategy PS2,   the condition of PS2 being complementary to PS1 is sufficient and necessary if there exists a  mixed strategy MS derived from PS1 and PS2 such that:  the expected runtime of MS is no more than that of PS1 for any initial population $X$,
 and less than that of PS1 for some initial population $X$.
To the best of our knowledge, no  similar  sufficient and necessary condition was rigorously established based on the runtime analysis of hybrid metaheuristics  before.  This is a step to understand  hybrid metaheuristics in theory.

Besides the above theoretical analysis, experiments are also implemented.  Experimental results  demonstrate that   mixed strategy EAs  may outperform pure strategy EAs on the 0-1 knapsack problem in up to 77.8\% instances.  In the experiments, the performance of an EA is measured by the fitness function value of the archive after 500 generations.  

It should be mentioned that a huge gap   exists between  empirical and  theoretical studies.   In experiments,  the optimal solution is usually unknown in most instances,    then  the expected runtime is unavailable;   in theory, it is difficult to analyse the best solution found in 500 generations or in any  fixed generations.  

\paragraph*{Acknowledgement}
This work is  supported by  the EPSRC under Grant EP/I009809/1, the National Natural Science Foundation of China under Grant  60973075  and Ministry of Industry and Information Technology  under Grant B0720110002.


\end{document}